\documentclass[11pt]{article}
\PassOptionsToPackage{hyphens}{url}
\usepackage[final]{acl}

\usepackage{times}
\usepackage{latexsym}
\usepackage[T1]{fontenc}
\usepackage[utf8]{inputenc}
\usepackage{microtype}
\usepackage{graphicx}
\usepackage{amsmath,amssymb,amsthm}
\usepackage{booktabs}
\usepackage{array}
\usepackage{multirow}
\usepackage{xcolor}
\usepackage{url}
\usepackage{enumitem}
\setlist{leftmargin=*}
\newcommand{\diag}{\textsc{RFR}}
\newcommand{\method}{\textsc{RFCR}}
\newcommand{\csicl}{\textsc{CS-ICL}}
\newcommand{\bbh}{\textsc{BBH}}
\newcommand{\rulename}{rule}
\newcommand{\atom}{atom}
\newcommand{\atoms}{atoms}
\newcommand\blfootnote[1]{%
\begingroup\renewcommand\thefootnote{}\footnote{#1}\addtocounter{footnote}{-1}\endgroup}
\newtheorem{proposition}{Proposition}
\newtheorem{lemma}{Lemma}

\newif\ifhidecomments
\ifhidecomments
    \newcommand{\chenhao}[1]{}
    \newcommand{\haokun}[1]{}
    \newcommand{\andrew}[1]{}
\else
    \newcommand{\chenhao}[1]{\textcolor{blue}{[\textsc{Chenhao}: #1]}}
    \newcommand{\haokun}[1]{\textcolor{green!30!brown}{[\textsc{Haokun}: #1]}}
    \newcommand{\andrew}[1]{\textcolor{teal}{[\textsc{Andrew}: #1]}}
\fi

\title{Robust Failure, Conservative Repair:\\
Textual Knowledge Distillation from Cross-Model Failures}
\author{Andrew Ren \and Haokun Liu \and Chenhao Tan \\ University of Chicago \\
  \texttt{\{renqy, haokunliu, chenhao\}@uchicago.edu}}

\begin{document}
\maketitle

\begin{abstract}
Failure-based textual knowledge distillation aims to discover gaps in a model’s knowledge by examining its task errors. The distilled knowledge can be useful for the reasoning of both this model (``source model'') and other models. However, this transfer of knowledge may not be stable. We define a rule atom to be a standalone rule injected into a model’s textual input at inference time. A rule atom can encode transferable task knowledge or model-specific reasoning patches that can confuse other models. Also, the injected rule atoms can be misapplied to unrelated cases, causing the model to incorrectly flip its answer based on irrelevant information. Building on a pipeline that distills training examples into task-specific cheat sheets that aid model reasoning, we examine when failure-derived rules can improve these cheat sheets. Our early experiment shows rule distillation from a single model's failures underperforms the baseline cheat sheet on non-source model families.
This motivates \textit{Robust Failure, Conservative Repair} (RFCR), a textual distillation procedure that derives rules from failures shared across models, sharpens their application boundaries using boundary cases, and abstains when no useful rule is found. On a 400-item BIG-Bench Hard \citep{suzgun2022bbh} task set, RFCR improves the baseline 
cheat sheets from 68.50\% to 71.25\% (+2.75 pp; 95\% CI [+1.25,+4.50]) without performance degradation on previously correct cases. Ablations and cross-model diagnostics support that accuracy gains come from both new knowledge injection and strict rule-application control.\blfootnote{Code and artifacts are available at: \url{https://github.com/ChicagoHAI/RFCR}.}
\end{abstract}

\begin{figure*}[t]
\centering
\small
\includegraphics[width=0.95\textwidth]
{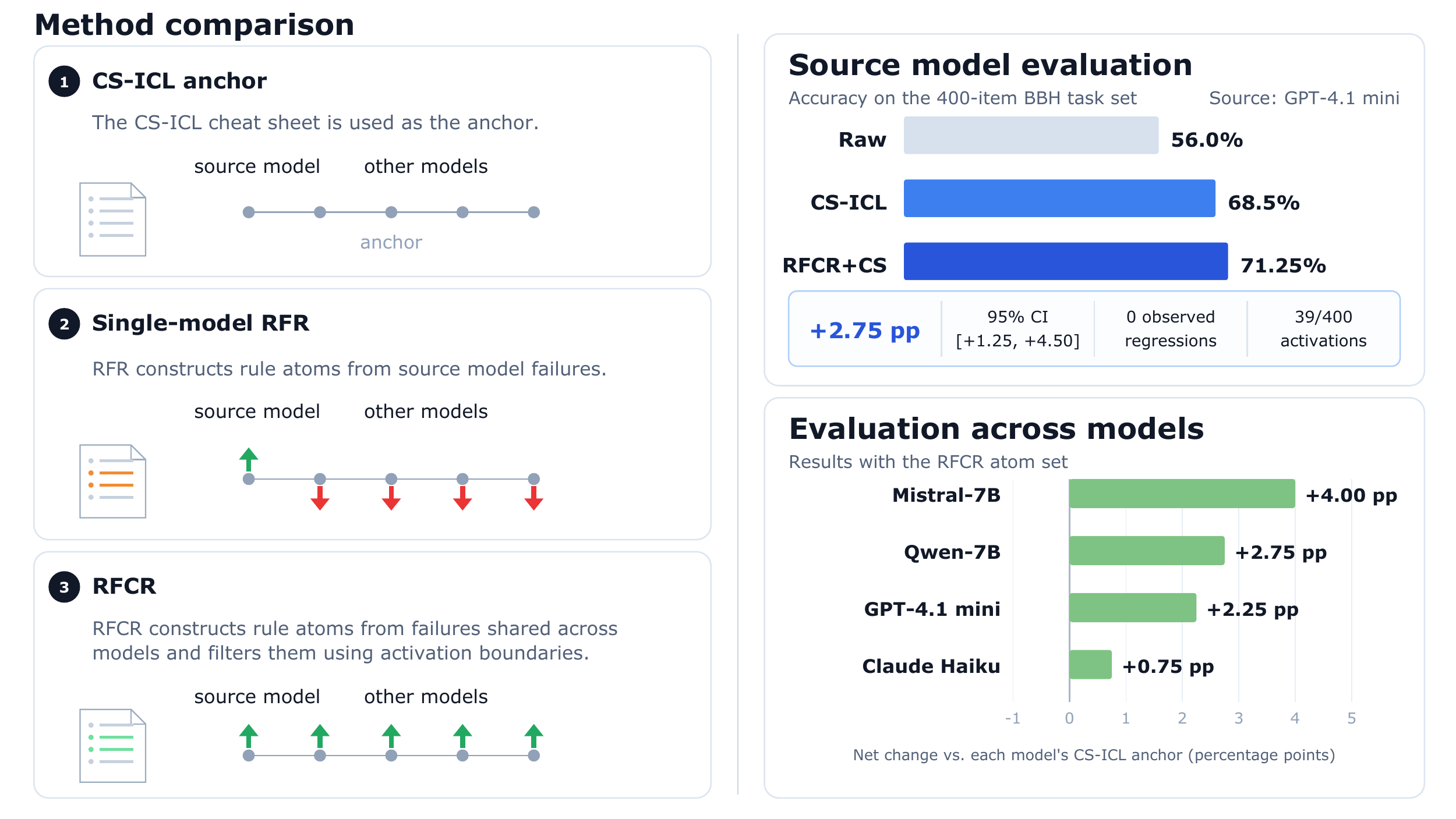}
\caption{Overview of \method{} and its results. \textbf{Left:} \csicl{} provides the anchor cheat sheet. Source-only failure refinement can produce model-specific patches that improve the source model but hurt other models. \method{} constructs rule atoms from failures shared across models and constrains their activation using boundary cases. \textbf{Right:} On the 400-item \bbh{} task set, \method{} improves \csicl{} by 11 items (+2.75 pp) without performance degradations on previously correct cases. These rule atoms also produce gains for the four models shown. Green and red arrows indicate improvements and performance degradations.}
\label{fig:overview1}
\end{figure*}

\section{Introduction}
Large language models (LLMs) can turn concrete examples and task heuristics into textual artifacts. One example is a cheat sheet: a compact set of task rules distilled from sample task cases and provided at inference time. Cheat sheet in-context learning (\csicl{}) demonstrates this capability by compressing many-shot demonstrations into a reusable task-specific cheat sheet \citep{honda2025csicl}. 
Similarly, the SAIR Mathematics Distillation Challenge asks whether strong mathematical reasoning can be compressed into a compact, human-readable cheat sheet that helps a weaker model on equational-theory problems \citep{sair2026mathdistillation}. In both cases, we construct concise and readable cheat sheets that can efficiently transfer strong-model reasoning to cheaper models. A natural question is whether such useful artifacts can be distilled from model failures.

This question is a form of textual knowledge distillation. We focus on cases that the source model answers incorrectly when receiving a baseline cheat sheet. These failures may expose missing knowledge. For example, if a source model repeatedly errs on a geometric relation, a pronoun-resolution pattern, or a logical form, it likely lacks knowledge needed to answer such questions correctly. The difficulty is that source-model failures do not automatically produce transferable task knowledge. A failure-derived rule may encode a model-specific patch instead of transferable knowledge. Such a patch may correct one model’s residual errors but hurt a target model that does not share those errors. Our Robust Failure Repair (\diag) diagnostic exhibits this behavior: source-only failure refinement underperforms the static \csicl{} baseline across all tested non-source model families.

We therefore introduce \textit{Robust Failure, Conservative Repair} (RFCR), a conservative textual distillation pipeline that derives rule atoms from failures shared across models. We define a rule atom to be a standalone rule injected into the model’s textual input at inference time. RFCR reframes failure-based refinement as \emph{selective textual distillation}. A candidate \rulename{} atom should not be accepted merely because it fixes a source-model failure. It should be accepted only when its application boundary limits its use on unrelated or already-correct cases. To measure this behavior, we introduce \emph{activation precision}. A rule atom is activated on a task case when the model uses it during reasoning. Activation precision is the proportion of cases on which the atom activates that belong to its intended target set. High activation precision means that a rule atom’s activations align closely with the target cases it is designed to solve. Low activation precision, on the other hand, indicates that the rule activates on irrelevant cases where its encoded knowledge is not needed. The resulting procedure is conservative: RFCR accepts a \rulename{} atom only when the evidence supports cross-model accuracy gains without degrading performance on other cases. 

In summary, we make the following contributions.
First, we show that rule atoms distilled from one model’s failures can correct the source model's errors but hurt other models that do not share the same errors. Second, we introduce activation precision to measure how frequently a \rulename{} atom is used on its intended cases rather than unrelated ones. Third, we develop \textit{Robust Failure, Conservative Repair} (\method{}). This conservative pipeline derives rule atoms from shared failures, constrains their activation with boundary cases, and leaves the cheat sheet unchanged if none passes validation. On a 400-item BIG-Bench Hard (\bbh) task set \citep{suzgun2022bbh}, \method{} improves \csicl{} by 11 items (+2.75 pp) without performance degradations on previously correct cases. More broadly, \method{} provides insight into a general conservative mechanism for failure-based knowledge distillation: find shared errors for rule construction, refine rule activation boundaries, and accept only rules that improve target accuracy without degrading other cases.

\section{Failure-Derived Distillation Can Produce Model-Specific Patches}
\label{sec:problem}
We first develop a Robust Failure Repair (\diag) diagnostic pipeline to explore how failure-based textual distillation can fail. Phase 0 builds a global cheat sheet from examples. Phase 1 patches the cheat sheet using source-model errors. Phase 2 adds local \textsc{Activate-If/Why} rules targeting residual source failures. Details of the \diag{} implementation can be found in Appendix~\ref{app:rfr}. As distillation moves from global examples toward residual source errors, the modification becomes more model-specific because it is refined on increasingly narrow error sets that other models may not share.

We then examine the impact of source-model-specific error fixing on non-source models. On five non-ceiling \bbh{} tasks, the full \diag{} pipeline falls below the static \csicl{} baseline for every non-source family (Table~\ref{tab:diag}). This pattern shows that source-only refinement is not just hard to transfer, it can actively interfere with target models, which leads to performance degradation.

\begin{table}[t]
\centering
\small
\begin{tabular}{lrrr}
\toprule
Target model & \csicl{} & \diag{} & $\Delta$ \\
\midrule
GPT-4.1 & 87.1 & 83.7 & -3.4 \\
Gemini-2.0 & 81.4 & 80.4 & -1.0 \\
Llama-3.3 & 78.3 & 75.7 & -2.6 \\
Claude-3.7 & 87.7 & 84.0 & -3.7 \\
\bottomrule
\end{tabular}
\caption{Early \bbh{} diagnostic: mean accuracy over five non-ceiling tasks. The full source-refined \diag{} prompt underperforms \csicl{} on every non-source family.}
\label{tab:diag}
\end{table}

\section{Transfer Theory for Failure-Derived Textual Knowledge Distillation}
\label{sec:theory}
We define the anchor cheat sheet to be the baseline cheat sheet we aimed to improve on. An anchor-correct task case is a task case that is evaluated correctly under the anchor cheat sheet. Let $C_0$ be an anchor cheat sheet. For a model $M$, let $h^{C_0}_M(x)$ be its answer to input $x$ under $C_0$. The anchor matters because the same \atom{} may help under a raw prompt but harm on top of a stronger cheat sheet. We therefore define anchor-conditioned regions
\begin{align}
F_M(C_0) &= \{x : h^{C_0}_M(x) \ne y\}, \\
K_M(C_0) &= \{x : h^{C_0}_M(x) = y\}.
\end{align}
A \rulename{} atom is a triple $m=(c,a_m,b_m)$: rule text $c$, a positive activation predicate $a_m(x)$, and an activation boundary predicate $b_m(x)$. Its effective activation is
\begin{equation}
\tilde a_m(x)=a_m(x)(1-b_m(x)).
\end{equation}
The rule-atom format can be found in Section~\ref{sec:construction}, and at inference time both predicates reduce to keyword matching against the input question (Section~\ref{sec:routing}).

Within the effective active region, three quantities determine the sign of the atom's effect:
\begin{align}
\pi_M(m;C_0) &= \Pr[x \in F_M(C_0) \mid \tilde a_m(x)=1], \\
f_M(m;C_0) &= \Pr[h^{C_0+m}_M(x)=y \mid \tilde a_m(x)=1,\nonumber\\
&\hspace{5.5em} x \in F_M(C_0)], \\
r_M(m;C_0) &= \Pr[h^{C_0+m}_M(x)\ne y \mid \tilde a_m(x)=1,\nonumber\\
&\hspace{5.5em} x \in K_M(C_0)].
\end{align}
Here $\pi_M(m;C_0)$ represents \emph{activation precision}: among examples where the \atom{} is used, how many are examples where the model needed help under the anchor? The terms $f_M$ and $r_M$ are conditional fix and performance degradation rates.

\begin{proposition}[Activation-precision threshold]
Assume $\Pr[\tilde a_m(x)=1] > 0$ and that $m$ does not change predictions outside its effective active region. The marginal accuracy effect of $m$ on model $M$ under anchor $C_0$ is positive if and only if
\begin{equation}
\pi_M(m;C_0) > \frac{r_M(m;C_0)}{f_M(m;C_0)+r_M(m;C_0)},
\end{equation}
with zero effect when $f_M(m;C_0)=r_M(m;C_0)=0$.
\end{proposition}

\noindent\textit{Proof.}
Only cases where $m$ is effectively active can change answer. Conditional on activation, the fraction $\pi_M(m;C_0)$ are anchor failures and can be fixed with probability $f_M(m;C_0)$. The remaining fraction $1-\pi_M(m;C_0)$ are anchor-correct examples and can be regressed with probability $r_M(m;C_0)$. Up to the positive factor $\Pr[\tilde a_m(x)=1]$, the expected accuracy change is
\begin{align}
&\pi_M(m;C_0) f_M(m;C_0) \nonumber\\
&\quad - (1-\pi_M(m;C_0)) r_M(m;C_0).
\end{align}
This quantity is positive exactly when the threshold above holds. \qed

This threshold explains why source gains do not imply transfer. A source-only gate can estimate whether $m$ is useful for $M_s$, but transfer depends on whether the same \atom{} activates precisely for a target $M_t$ under the target anchor.

\begin{proposition}[Source acceptance is not enough]
For any $f,r \in (0,1]$, there exist a source model $M_s$, a target model $M_t$, and a \rulename{} atom $m$ such that $m$ helps $M_s$ and hurts $M_t$, even when both models have the same conditional fix rate $f$ and performance degradation rate $r$.
\end{proposition}

\noindent\textit{Proof.}
Choose $\pi_s > r/(f+r)$ and $\pi_t < r/(f+r)$. Construct active regions where the fraction of baseline failures is $\pi_s$ for the source and $\pi_t$ for the target, while both models share the same $f$ and $r$. Proposition 1 gives a positive effect for $M_s$ and a negative effect for $M_t$. \qed

\begin{proposition}[Anchor-specific validity]
A \rulename{} atom can have positive marginal gain under one anchor $C_0$ and negative marginal gain under another anchor $C_1$ for the same model $M$.
\end{proposition}

\noindent\textit{Proof sketch.}
Choose anchors such that the atom's effective active region contains a high fraction of failures under $C_0$ and a low fraction of failures under $C_1$. The activation-precision threshold gives opposite signs. This formalizes why \atoms{} selected on a raw prompt must be revalidated before being deployed on top of a stronger \csicl{} anchor. \qed

\begin{proposition}[No-op safety]
We define no-op as the deliberate choice of the pipeline to not proceed with any modification of the anchor cheat sheet. If inactive examples receive the anchor cheat sheet only, then a routed \rulename{} \atom{} cannot change predictions outside its effective activation region. 
\end{proposition}

\noindent\textit{Proof.}
By construction, the inference procedure uses $h_M^{C_0}(x)$ whenever no selected \atom{} has $\tilde a_m(x)=1$. Thus non-activated examples are unchanged. If no \atom{} activates, all examples follow the anchor. \qed

\paragraph{Failure overlap as a risk screen.}
Let $F_s$ and $F_t$ be source and target failure sets under $C_0$. 
We define the source-private failure region as the set of task cases where only the source model made a failed prediction and the shared failure region as the set of task cases where both source and target model give false answers. The shared failure region is $F_s \cap F_t$. The source-private failure region is $F_s \setminus F_t$. We summarize their overlap with
\begin{equation}
\rho_{s\to t}=\frac{|F_s\cap F_t|}{|F_s|}, \quad J(s,t)=\frac{|F_s\cap F_t|}{|F_s\cup F_t|}.
\end{equation}
Low overlap warns that source-derived rules are likely to activate imprecisely for the target. High overlap is not a guarantee, because wording, routing, fix rate, and performance degradation risk still matter.

\begin{lemma}[Overlap bound]
Let $\tilde A_m = \{x : \tilde a_m(x)=1\}$ be the effective active region of atom $m$. If $\tilde A_m \subseteq F_s$, then
\begin{align*}
\pi_{M_t}(m;C_0) &= \Pr[x \in F_t(C_0) \mid \tilde a_m(x)=1] \\
&\le \frac{|F_s \cap F_t|}{|\tilde A_m|}.
\end{align*}
If $\tilde A_m = F_s$, then $\pi_{M_t}(m;C_0)=\rho_{s\to t}$.
\end{lemma}

\noindent\textit{Proof.}
Since $\tilde A_m \subseteq F_s$, we have $\tilde A_m \cap F_t \subseteq F_s \cap F_t$. Dividing by $|\tilde A_m|$ gives the bound. If $\tilde A_m=F_s$, the expression becomes $|F_s \cap F_t|/|F_s|$. \qed

\section{\textit{Robust Failure, Conservative Repair}}
\label{sec:method}

\begin{figure*}[t]
\centering
\small
\includegraphics[width=0.85\textwidth]
{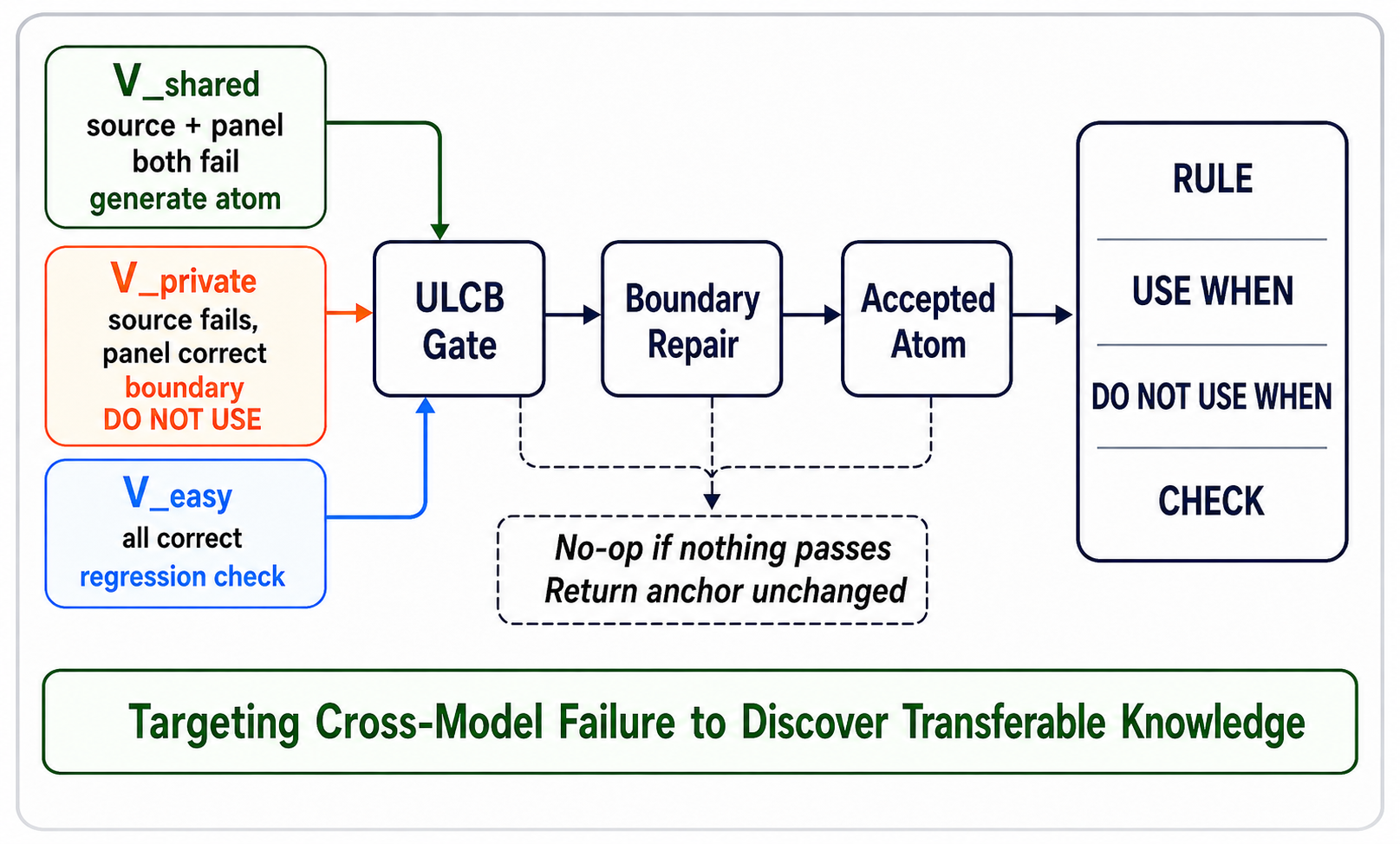}
\caption{\method{} treats a failure-derived atom as rule text plus an activation predicate and an activation boundary predicate. If the gate fails or no route fires, the system returns the anchor output.}
\label{fig:overview2}
\end{figure*}

\textit{Robust Failure, Conservative Repair} (\method) is a conservative textual knowledge distillation procedure. It treats failure-derived knowledge distillation as a routed repair procedure rather than as a global prompt update. We release our code and artifacts at \url{https://github.com/ChicagoHAI/RFCR}.

\subsection{Rule Atom Construction}
\label{sec:construction}
Let $P$ be a proxy set of models used during refinement, excluding the held-out target family when such a target is defined. Under $C_0$, the set of validation examples $V$ are divided into
\begin{align}
V_{\mathrm{shared}} &= F_s \cap \left(\bigcup_{M_j\in P} F_j\right), \\
V_{\mathrm{private}} &= F_s \cap \left(\bigcap_{M_j\in P} K_j\right), \\
V_{\mathrm{easy}} &= K_s \cap \left(\bigcap_{M_j\in P} K_j\right).
\end{align}
\method{} generates candidate \atoms{} from $V_{\mathrm{shared}}$. It uses source-private failure examples ($V_{\mathrm{private}}$) not as positive evidence but as boundary cases that tell the system when not to apply a rule. Easy examples ($V_{\mathrm{easy}}$) are used to check for any rule-\atom{}-induced performance degradation. The subset $V \backslash (V_{\mathrm{shared}} \cup V_{\mathrm{private}} \cup V_{\mathrm{easy}})$ accounts for the cases where the source model provides correct reasoning but at least one proxy model fails. We put this subset aside intentionally as the source model's fix rate for cases in this set remains $0$ regardless of the quality of candidates generated from this subset. This makes the conservative acceptance gate defined below too strict for any rule acceptance.

\paragraph{Rule-atom format.}
Each accepted \rulename{} atom is compact and human-readable:
\begin{quote}\small
\textbf{RULE:} one task-level decision principle.\\
\textbf{USE WHEN:} the minimal structural trigger.\\
\textbf{DO NOT USE WHEN:} source-private or confounding cases.\\
\textbf{CHECK:} one final verification step.
\end{quote}

\paragraph{Atom generation.}
For rule atom generation, we divide the training split into two disjoint sets: a rule-generation set (60\%) and a gate set (40\%). Failure regions are computed and candidate \atoms{} are generated on the rule-generation set only. The gate set is reserved for the acceptance test below, so no item is both used in rule generation and validation.  We score the source and proxy panel models on the generation set using the baseline cheat sheet and group the shared failures by failure patterns. For each cluster, the source model receives the clustered failure examples together with source-private examples as explicit negative boundary cases, and is asked to write a rule in the four-field format above. Several candidates are generated per cluster at different temperatures, and we run each candidate through the validation set and decide acceptance through the acceptance gate. The full generation prompt is in Appendix~\ref{app:impl}.

\paragraph{Conservative acceptance.}
Each candidate \atom{} $m$ is evaluated based on a conservative utility score:
\begin{align}
U(m) =&\ \min_{j\in P} \mathrm{LCB}_{95}[\Delta_j(V_s;m)] \\
&- \lambda \max_{j\in P} \mathrm{UCB}_{95}[R_j(V_p;m)] \\
&- \mu \max_{j\in P} \mathrm{UCB}_{95}[R_j(V_e;m)] - \nu |c|.
\end{align}
Here $V_s,V_p,V_e$ are the shared, private, and easy regions defined above. $\Delta_j(V_s;m)$ is the fix rate of $m$ on $V_{\mathrm{s}}$. $R_j(V_p;m)$ and $R_j(V_e;m)$ are degradation rates on $V_{\mathrm{p}}$ and $V_{\mathrm{e}}$. We use \textbf{Wilson score intervals} with $z{=}1.96$ for the confidence bounds and $\lambda{=}\mu{=}1.0$ and $\nu{=}0.05$ throughout the experiments (bound computation in Appendix~\ref{app:impl}). The first term requires shared-failure improvement. The next two penalize performance degradations on source-private failures and easy examples. The length penalty discourages broad memories. A rule is accepted only when $U(m)>0$. In the strict main-table deployment we also require positive net improvement and no observed performance degradations against \csicl{}-correct items.  

\paragraph{Activation boundary repair and data separation.}
Candidate \atoms{} can pass conservative validation while being incorrectly triggered on irrelevant cases. \method{} therefore includes a boundary-repair pass over activation predicates. This procedure fine-tunes the triggers using source-private failures with anchor-correct examples as negative controls. It is worth noting that the pass changes only whether an accepted \atom{} is allowed to fire, but not the content of the \atom{}. We first conduct a trial run on the items with the accepted \atoms{} active to observe which items are fixed or degraded and then modify the trigger accordingly before freezing them for the final rerun. Since the repair procedure utilizes the test items, the 400-item result should be read as a controlled development result rather than a held-out test. We therefore report the test result before and after repair as an ablation result. For held-out validation, we report the frozen knowledge atom set's performance on 295 official BIG-Bench items unseen during the rule atom generation and boundary repair phases.

\subsection{Rule Atom Routing at Inference}
\label{sec:routing}
At inference time, the model receives $C_0$ plus the \atoms{} whose activation condition fires and whose activation boundary condition does not fire on the task case. If no \atom{} activates, the input is exactly the unchanged \csicl{} anchor. Atom activation is determined by keyword matching between the activation predicates (USE WHEN / DO NOT USE WHEN section) of the rule atom and task case question text. We define this mode as routing in comparison to the global mode. Under global mode, all accepted atoms are presented to each task case with activation decision relies entirely on LLM's interpretation of the rule's activation predicates at inference time.

\paragraph{No-op deployment.}
If no \atom{} passes route and activation boundary checks, \method{} returns exactly the \csicl{} output. Abstention is part of our method: the system should not add a risky \rulename{} merely because a source model improved during development.

\begin{table*}[t]
\centering
\small
\begin{tabular}{lrrrrrrr}
\toprule
Task Set Evaluated & $n$ & Raw & \csicl{} & \method{} & $\Delta$ & 95\% CI & Reg. \\
\midrule
\bbh{} main eligible \atom{} tasks & 400 & 56.00 & 68.50 & 71.25 & +2.75 & [+1.25,+4.50] & 0 \\
\bbh{} held-out cases & 295  & -- & 66.10 & 67.46 & +1.36 & [+0.34,+2.71] & 0 \\
\bbh{} full18, supplementary & 1716 & 70.75 & 75.58 & 76.22 & +0.64 & [+0.29,+1.05] & 0 \\

\bottomrule
\end{tabular}
\caption{Strict unified-evaluator results after disambiguation activation boundary repair. $\Delta$ is absolute accuracy improvement in percentage points over the GPT-4.1-mini \csicl{} anchor. The full18 row is supplementary because most tasks are inactive passthrough.}
\label{tab:main}
\end{table*}

\begin{table*}[t]
\centering
\small
\resizebox{\textwidth}{!}{%
\begin{tabular}{lrrrrrrr}
\toprule
Model & $n$ & CS-ICL & \method{}+CS & Fix & Reg. & Net & $\Delta$ pp / 95\% CI \\
\midrule
\texttt{anthropic/claude-haiku-4.5} & 400 & 75.75 & 76.50 & 4 & 1 & +3 & +0.75 [-0.25, 2.00] \\
\texttt{google/gemini-2.5-flash-lite} & 400 & 64.25 & 63.75 & 3 & 5 & -2 & -0.50 [-2.00, 1.00] \\
\texttt{mistral-7b-instruct-v0.3} & 400 & 37.25 & 41.25 & 16 & 0 & +16 & +4.00 [2.25, 6.00] \\
\texttt{qwen2.5-7b-instruct} & 400 & 51.00 & 53.75 & 13 & 2 & +11 & +2.75 [1.00, 4.75] \\
\bottomrule
\end{tabular}}
\caption{Cross-model rerun with the repaired \atom{} set. Rule atoms are applied without target-model route tuning (i.e. visible during all task case evaluations with soft activation boundary check by LLM at inference time). The evidence supports positive-net transfer on several families, but not a cross-model zero-degradation guarantee.}
\label{tab:crossmodel}
\end{table*}

\section{Experiments and Results}
\label{sec:experiments}
\paragraph{Evaluation Task Sets.}
The main evaluation uses \bbh{} tasks for which the \method{} pipeline produced main-eligible routed \atoms{}: disambiguation QA, formal fallacies, geometric shapes, and object counting. This main task set contains 400 examples. We also report a supplementary 18-task \bbh{} task set as a sanity check on scale and a 295 official BIG-Bench upstream examples outside the seen train and test BBH subsets for validation. The 295 held-out set consists of 8 disambiguation QA, 100 formal fallacies, 87 geometric shapes, and 100 object-counting examples. The selection of these examples did not rely on any prior knowledge of the performance or expected outcomes of either our method or the baselines.

\paragraph{Anchor and protocol.}
The anchor is a GPT-4.1-mini-generated \csicl{} \cite{honda2025csicl} cheat sheet, scored by GPT-4.1-mini in the unified evaluator. For every inactive item, the \method{} condition copies the \csicl{} output exactly. The strict rerun uses the same task set, candidate rule atoms, prompt builder, parser, retry path, scoring guard, low-concurrency call schedule, and 10,000 paired-bootstrap samples across conditions.

\paragraph{Models and access.}
 We choose \texttt{gpt-4.1-mini} as our source model. This model is used in the generation of the \csicl{} baseline cheat sheets as well as all candidate \atoms{} generation. The proxy models we choose are \texttt{qwen2.5-7b-instruct} and \texttt{mistral-7b-instruct-v0.3} with approximately 7B parameters each. These models are hosted through vLLM on a compute cluster. Models used for cross-model transfer evaluation (\texttt{claude-haiku-4.5}, \texttt{gemini-2.5-flash-lite}) and models used for the early \diag{} diagnosis (GPT-4.1, Gemini-2.0-Flash, Llama-3.3-70B, Claude-3.7-Sonnet) are accessed through OpenRouter.

\paragraph{Computational budget.}
The \method{} pipeline using \csicl{} cheat sheet as baseline takes around 35{,}000 to 45{,}000 model calls. In our case since we are hosting two proxy models through vLLM, the pipeline uses around 25{,}000 paid API calls, which costs around 20 USD. The two proxy models took 48 GPU hours in total. We also report the per stage token usage. Each \csicl{} cheat sheet generation process consumes 149{,}683 input tokens and 9{,}122 output tokens. Each evaluation pass of the 400-item set takes 955{,}275 input tokens and 1{,}290 output tokens.

\paragraph{Main controlled result.}
Table~\ref{tab:main} gives the strict source-protocol result. On the 400-item \bbh{} task set, \method{} improves the \csicl{} baseline from 68.50\% to 71.25\%, a gain of +11 items or +2.75 pp, with a 95\% confidence interval of [+1.25,+4.50] pp and no paired performance degradations.

Regarding held-out confirmations, we conducted an additional evaluation using the same frozen knowledge atom set constructed for the main result (plus the CS-ICL baseline). We use the same evaluation pipeline as in the main result for consistency. Out of the 295 test cases, we observed 16 cases with atom activation, while the other 279 were CS-ICL passthrough. On accuracy results, we see that the CS-ICL baseline scored $66.10\%$, and our overlay result gave $67.46\%$. This $+1.36$ pp improvement results from four fixes and zero observed regressions.

\paragraph{Full 18-task breakdown.}
We also report the result of \method{} on the 18-task BBH set (Table~\ref{tab:eligibility} in Appendix~\ref{app:audits}). The supplementary 18-task check improves from 75.58\% to 76.22\%, also without performance degradations (Table~\ref{tab:main}). We observe that for 14 out of the 18 tasks \method{} refrains from adding any atomic rules upon the \csicl{} baseline. The three tasks that \method{} does generate valid atomic rules are formal fallacies, disambiguation QA, and geometric shapes, which constitute the task sets we use for our main results. The object counting task is also included in the main task set. However, its candidate \atom{} was rejected by the acceptance gate, so it runs as passthrough of the \csicl{} baseline. It is worth noting that the 18-task run is a sanity check for inactive passthrough rather than a benchmark-wide improvement claim.

\begin{table}[t]
\centering
\scriptsize
\resizebox{\columnwidth}{!}{%
\begin{tabular}{lrrcrr}
\toprule
Method & Acc. & $\Delta$ & 95\% CI & Fix & Reg. \\
\midrule
\csicl{} anchor (seed) & 68.50 & -- & -- & -- & -- \\
ProTeGi & 70.50 & +2.00 & [$-$1.50,+5.50] & 32 & 24 \\
GEPA & 65.50 & $-$3.00 & [$-$6.25,+0.25] & 16 & 28 \\
\method{} & \textbf{71.25} & \textbf{+2.75} & [+1.25,+4.50] & 11 & \textbf{0} \\
\bottomrule
\end{tabular}%
}
\caption{Prompt-optimization baselines on the same 400-item main task set.  ProTeGi and GEPA start from the exact \csicl{} frozen cheat sheet as we used for \method{}.}
\label{tab:baselines}
\end{table}

\paragraph{Prompt-optimization baselines.}

To evaluate \method{} as a prompt optimization method, we ran two prompt optimization methods, ProTeGi \cite{pryzant2023automatic} and GEPA \cite{agrawal2025gepa}. Both baselines use the same training and testing sets as in our \method{} main results. Both methods use gpt-4.1-mini as the source model and initiate with the same \csicl{} cheat sheet as \method. We use the default parameters of both methods. The results show that among the 400 test examples, ProTeGi achieved 32 fixes with 24 regressions. GEPA got 16 fixes but 28 regressions. Overall, we see that these optimization methods yielded more fixes than our RFCR pipeline at the cost of causing many regressions along the way. In GEPA’s case, the revision-induced regression outnumbered the fixes it introduced, making it perform worse than the CS-ICL baseline. This result echoes our concern about the uncertainty of prompt optimization’s effects on already correct cases and emphasizes the value of the strict knowledge atom selection pipeline proposed by RFCR. More details of the baseline method execution are in Appendix~\ref{app:baselines}.

\paragraph{Cross-model diagnostics.}
Table~\ref{tab:crossmodel} evaluates the repaired \atom{} set on model families. The repaired \atoms{} have positive net transfer on Claude Haiku, Mistral, and Qwen. We further observe that models with relatively weaker reasoning capability, such as Qwen and Mistral, show larger net gains with relatively low performance degradation. This supports our interpretation that \method{} distills task-specific knowledge that transfers broadly rather than merely patching source-model errors. However, Gemini Flash Lite has negative net transfer and several models show nonzero performance degradations, indicating that even task-specific rules may not universally apply across all model families.

\begin{table}[t]
\centering
\scriptsize
\resizebox{\columnwidth}{!}{%
\begin{tabular}{lrrrrrr}
\toprule
Task & \csicl{} & \method{}+CS & Act. & Fix & Reg. & Net \\
\midrule
Disambig. QA & 83 & 87 & 15 & 4 & 0 & +4 \\
Formal fallacies & 67 & 67 & 4 & 0 & 0 & 0 \\
Geometric shapes & 56 & 63 & 20 & 7 & 0 & +7 \\
Object counting & 68 & 68 & 0 & 0 & 0 & 0 \\
\midrule
Total & 274 & 285 & 39 & 11 & 0 & +11 \\
\bottomrule
\end{tabular}%
}
\caption{Task-level decomposition on the 400-item main task set. Counts are out of 100 items per task. ``Act.'' is the number of activated routed-\rulename{} applications.}
\label{tab:task}
\end{table}

\begin{table*}[t]
\centering
\small
\begin{tabular}{lrrrrrrr}
\toprule
Version & $n$ & Activated & Fix & Reg. & Net & Acc. & $\Delta$ vs CS \\
\midrule
Original locked package & 400 & 59 & 9 & 0 & +9 & 70.75 & +2.25 \\
Repaired route & 400 & 39 & 11 & 0 & +11 & 71.25 & +2.75 \\
\bottomrule
\end{tabular}
\vspace{0.35em}

\begin{tabular}{lrrrr}
\toprule
Disambiguation route version & Activated & Fix & Reg. & Net \\
\midrule
Original broad route & 35 & 1 & 0 & +1 \\
Repaired narrow route: \texttt{needed\_one OR told\_that} & 15 & 4 & 0 & +4 \\
\bottomrule
\end{tabular}
\caption{CS-ICL-anchored same-cache route ablation. The repaired activation trigger here is Disambiguation-QA-specific and presented as an example of boundary repair rather than a generalizable mechanism.}
\label{tab:repair}
\end{table*}

\begin{table*}[t]
\centering
\small
\begin{tabular}{lrrrrrrr}
\toprule
Condition & $n$ & Acc. & Correct & Fix & Deg. & Net & Activated \\
\midrule
No \atom{} & 400 & 56.00 & 224 & 0 & 0 & 0 & 0 \\
Global \atom{} & 400 & 66.00 & 264 & 78 & 38 & +40 & 400 \\
Raw-Routed \atom{} & 400 & 66.00 & 264 & 43 & 3 & +40 & 132 \\
Boundary-routed \atom{} & 400 & 69.50 & 278 & 54 & 0 & +54 & 132 \\
\bottomrule
\end{tabular}
\caption{Rule Atom Activation Control Ablation with no Baseline Cheat Sheet}
\label{tab:source_only}
\end{table*}

\paragraph{Task-level decomposition.}
Table~\ref{tab:task} shows that the gain is sparse. Geometric shapes contributes +7 items and disambiguation QA contributes +4. Formal fallacies is neutral; object counting is a no-op under the repaired \csicl{} anchor. This sparse pattern is consistent with our methodology: \method{} improves only where a routed repair is available and otherwise leaves the anchor unchanged.

\paragraph{Boundary-repair ablation.}
Because the main result uses boundary repaired routing for Disambiguation QA items, we investigate exactly how this boundary repair effort helps in test accuracy improvement. Table~\ref{tab:repair} compares the original locked \atom{} set with the boundary repaired set with the same \csicl{} baseline and test set. The repaired route improves the test accuracy to 71.25\% without performance degradations. The boundary-repaired routing reduces the number of activations on test cases from 35 to 15 but increases the number of fixes from 1 to 4. The result shows that narrowing the activation space improves model performance by directing each \atom{} only to cases where its encoded knowledge is relevant. This supports the activation-precision view: sharpening a rule's activation boundary can increase fixes without increasing performance degradations.

\paragraph{Global versus routed exposure.}
We also examine the effect of different activation control measures on the same kinds of \atoms{} when baseline cheat sheets are not presented. Table~\ref{tab:source_only} shows that global exposure leads to high activation frequency of the \atoms{} but also leads to a substantial performance degradation. The raw-routing mode (LLM receives the full atom list at inference time, but task results of cases not activating any atoms fall back to the no atom result) reduces the degradation rate drastically. Boundary-routed mode (what standard \method{} uses) further improves the accuracy result by boosting activated task case accuracy and eliminating degradation.

\begin{table}[t]
\centering
\scriptsize
\setlength{\tabcolsep}{3pt}
\begin{tabular}{lrrrrr}
\toprule
Variant & Acc. & Fix & Reg. & Net & Act. \\
\midrule
CS anchor & 68.50 & 0 & 0 & 0 & 0 \\
Gen. only & 70.75 & 14 & 5 & +9 & 79 \\
Gen.+ULCB & 70.75 & 9 & 0 & +9 & 59 \\
+Boundary & 71.25 & 11 & 0 & +11 & 39 \\
\bottomrule
\end{tabular}
\caption{Generated-only versus activation-boundary-repaired \atoms{}.}
\label{tab:generated_boundary}
\end{table}

\paragraph{Generated rules versus activation boundary repair.}
Table~\ref{tab:generated_boundary} separates candidate generation from conservative validation. Generation-only approach fixes more cases but introduces five paired degradations. The upper/lower confidence bounds (ULCB) filtering show no performance degradations with the same net accuracy gain. The repaired activation boundary approach further improves on both net gain and activation precision. This matches our claim that the main result comes from sparse conservative routing instead of indiscriminately adding more generated text.

\section{Analysis}
\paragraph{Sparse gains.}
The main gain of \method{} is sparse and routed across subsets of cases with similar demands for specific knowledge, rather than a global marginal improvement over the entire task. Only 39 item-rule applications activate on the 400-item task set, yet they account for +11 net fixes without any source-model performance degradations. This aligns with our Proposition 1: a valuable \atom{} needs high activation precision on the subset where it fires, not coverage of every item.

\paragraph{Activation precision.}
Table~\ref{tab:activation_precision} reports \atom-level activation precision and net gain. Fix precision here refers to the number of fixed cases among all activated cases, while activation precision measures the fraction of CS-ICL failures among activated cases. The geometric \atom{} has both high activation precision and the largest net gain. The formal-fallacies \atom{} is safe but so underactive that it does not sway the test result either way. We see here that fix precision and activation precision are both positively correlated with the net gain, while the total number of activated cases does not.
\begin{table}[t]
\centering
\scriptsize
\resizebox{\columnwidth}{!}{%
\begin{tabular}{lrrrrr}
\toprule
Atom & Act. & Fix. prec. & Act. prec. & Fix & Net \\
\midrule
Disambig-v1.2 & 15 & 26.7 & 40.0 & 4 & +4 \\
FF-neg. & 4 & 0.0 & 50.0 & 0 & 0 \\
Geo-1 & 8 & 12.5 & 12.5 & 1 & +1 \\
Geo-6 & 12 & 50.0 & 91.7 & 6 & +6 \\
\bottomrule
\end{tabular}}
\caption{Activation-precision diagnostics.}
\label{tab:activation_precision}
\end{table}

\paragraph{Abstention is part of the method.}
Object counting has eligible task coverage but zero active applications in the repaired main table. This is a no-op decision made by the pipeline rather than a failure to deploy. The supplementary full-18 result also shows that inactive passthrough does not introduce any performance degradation.

\paragraph{Cross-model performance degradations are informative.}
Cross-model performance degradations show the weakness of source-derived activation boundary. The same \atom{} can have positive net transfer and still regress target-correct items if the target activation boundary differs from the source activation boundary. This motivates leave-one-family-out proxy gating and per-target activation boundary diagnostics.

\section{Related Work}
\paragraph{Textual knowledge distillation and cheat sheets.}
\csicl{} distills many-shot in-context examples into a compact textual summary, reducing inference-time prompt cost while preserving much of the benefit of long demonstrations \citep{honda2025csicl}. Dynamic Cheatsheet studies persistent, evolving memory for black-box language models at test time \citep{suzgun2026dynamic}. Our work is complementary: rather than asking only whether a distilled \rulename{} is useful in aggregate, we ask when failure-derived textual knowledge transfers across model families without becoming a model-specific patch.

\paragraph{Prompt optimization and reflective search.}
Prompt optimization methods use data, feedback, search, evolution, or textual gradients to improve model inputs and behavior \citep{pryzant2023automatic,fernando2023promptbreeder,yang2024optimizers,khattab2024dspy,yuksekgonul2025textgrad,agrawal2025gepa}. These methods usually optimize a prompt for a measured development distribution. \method{} focuses on a narrower distillation risk: a failure-derived \atom{} can be accepted by the source model yet hurt another model if it activates on examples that the target already solves.

\paragraph{Reasoning benchmarks.}
We evaluate on BIG-Bench Hard (\bbh), a challenging subset of BIG-Bench designed to stress reasoning behavior \citep{srivastava2022bigbench,suzgun2022bbh}. Because several \bbh{} tasks are sensitive to prompt wording and reasoning format, they are useful for testing whether distilled textual repairs are robust rather than source-specific. We discuss AGIEval \citep{zhong2024agieval} as a diagnostic benchmark, but do not use it as a main claim until a locked nonnegative-anchor protocol is available.

\section{Conclusion}
Rules distilled from one model’s failures can improve source model's accuracy but hurt other models' performance when applied to unrelated or already-correct cases. We use activation precision to measure how often a rule is applied to the cases it was designed to address. \method{} turns this analysis into a procedure by constructing \rulename{} \atom{}s from shared failures, refining activation boundaries with source-private failures and anchor-correct examples, and applying conservative acceptance checks to candidate rule atoms. On a 400-item \bbh{} task set, \method{} improves the \csicl{} baseline accuracy by 2.75 percentage points without degrading any previously correct case. The success of \method{} suggests that improving cheat sheets through textual knowledge distillation requires both extracting new knowledge and controlling when that knowledge is applied. Future work should explore forms of cheat sheet modification beyond simply appending rules to separate the effects of cheat sheet structure from those of newly distilled knowledge.

\section*{Ethical Considerations}
This work studies textual distillation on reasoning benchmarks. The main risk is overclaiming reliability: a \rulename{} artifact that improves one model may harm another. We therefore report paired performance degradations, preserve a no-op fallback, and emphasize that source-model gains alone are not evidence of general improvement. The experiments do not require personal data.

\section*{Limitations}
The main improvement is modest and concentrated in geometric shapes and disambiguation QA. The evidence supports conservative routed textual distillation as a sparse no-degradation improvement over a strong same-generator \csicl{} anchor instead of a broad benchmark optimization claim.

The zero-degradation property holds only for the locked source-model protocol. Cross-model performance degradations remain nonzero for some targets, reflecting that activation predicates were derived from source-model failures and not validated against every target-model activation boundary. The paper includes a retrospective leave-one-family-out proxy-gating diagnostic, but it does not yet regenerate \atoms{} with each family fully excluded. A stronger test should perform end-to-end leave-one-family-out generation, gating, freezing, and target-only evaluation.

The main evaluation task set contains four \bbh{} tasks with eligible routed \atoms{}. Broader benchmark-scale claims require more tasks, learned routers, and stronger held-out validation. The supplementary full-18 result is useful as a sanity check, but it should not be interpreted as evidence that every task benefits from \method{}, because most items use inactive passthrough.

The main 400-item result is a controlled development result, not a fresh held-out test after activation boundary repair. The final disambiguation activation boundary was repaired using item-effect evidence from the \atom{}-enabled pilot before the repaired route was frozen and rerun under a strict evaluator.

\section*{Acknowledgments}
We thank the anonymous reviewers for their insightful comments. We also thank the NeuriCo project from Chicago Human+AI Lab for its assistance in early-stage mechanistic exploration. This work was supported in part by compute credits from Modal, provided through CMSC 25750/35750 Large Language Models at the University of Chicago, as well as the University of Chicago AI initiatives.

\bibliography{custom}

\appendix
\section{Reproducibility Statement}
 We release our code, the frozen knowledge atom set, and the \csicl{} cheat sheets used in all experiments at \url{https://github.com/ChicagoHAI/RFCR}. The final run uses a fixed task set, fixed \csicl{} cache, fixed selected-\atom{} file, the same parser and prompt builder across conditions, a scoring guard, low-concurrency calls, inactive passthrough for non-activated items, and a paired bootstrap with 10,000 samples. The reported \csicl{} column is the unified-evaluator GPT-4.1-mini \csicl{} anchor/cache condition and should not be mixed with separately reported deployment numbers from earlier runs. The main 400-item result and all ablation tables are from a single locked run. The exploratory holistic rewrite table (Appendix~\ref{app:holistic}) reports means over three seeds (1000, 2000, 3000). Cross-model 95\% CIs are computed per model using the same paired-bootstrap procedure.

\section{RFCR Implementation Details}
\label{app:impl}

\paragraph{Shared-failure identification.}
A failure is classified as \emph{shared} if at least one proxy model $M_j \in P$ also fails the item under $C_0$. This is a union criterion: $V_{\mathrm{shared}} = F_s \cap (\bigcup_j F_j)$. A failure is \emph{source-private} if \emph{every} panel model gets it right: $V_{\mathrm{private}} = F_s \cap (\bigcap_j K_j)$. Items that all models including the source get right form $V_{\mathrm{easy}}$. Cluster size within $V_{\mathrm{shared}}$ serves as a soft priority signal for which failure patterns to write \atoms{} for first.

\paragraph{Atom generation prompt.}
Atoms are generated by presenting the source model with a curated batch of shared-failure examples together with a selection of source-private boundary examples, then issuing the following template:

\begin{quote}\small
\texttt{You are given examples where a language model failed despite having a general cheat sheet. Identify the common structural pattern and write a repair rule.}\\[2pt]
\texttt{Failed examples (model got wrong):}\\
\texttt{[SHARED\_FAILURE\_EXAMPLES]}\\[2pt]
\texttt{Boundary examples (model gets right — the rule must NOT fire here):}\\
\texttt{[SOURCE\_PRIVATE\_EXAMPLES]}\\[2pt]
\texttt{Write the rule in this exact format:}\\
\texttt{RULE: <one task-level decision principle>}\\
\texttt{USE WHEN: <minimal structural trigger>}\\
\texttt{DO NOT USE WHEN: <conditions from the boundary examples>}\\
\texttt{CHECK: <one final verification step>}\\[2pt]
\texttt{Be as specific as possible. Do not write a rule that applies to all task examples.}
\end{quote}

The \textbf{DO NOT USE WHEN} field is seeded from $V_{\mathrm{private}}$ so the model is nudged toward writing a narrow activation boundary condition rather than a broad topic match. Multiple candidates are drawn per cluster using temperature sampling; all candidates are then scored by the ULCB gate.

\paragraph{Upper/Lower confidence bounds.}
Confidence bounds are Wilson score intervals at $z{=}1.96$. $\mathrm{LCB}_{95}[\Delta_j]$ is the Wilson lower bound of the observed fix proportion on the gate subset, and $\mathrm{UCB}_{95}[R_j]$ is the Wilson upper bound of the observed performance degradation proportion. An \atom{} passes the gate only when $U(m) > 0$. The strict deployment condition additionally requires zero \emph{observed} performance degradations against \csicl{}-correct items across all panel models (not just a UCB bound).

\paragraph{Activation Boundary repair prompt.}
After an \atom{} passes the ULCB gate, its activation predicate is inspected for over-triggering. The repair prompt presents the source model with over-triggered cases and asks for a narrower condition:

\begin{quote}\small
\texttt{An accepted rule fires on the following examples but should not (anchor-correct or source-private cases):}\\
\texttt{[OVER\_TRIGGERED\_EXAMPLES]}\\[2pt]
\texttt{The rule correctly fires on these examples:}\\
\texttt{[CORRECTLY\_TRIGGERED\_EXAMPLES]}\\[2pt]
\texttt{Current USE WHEN: [CURRENT\_TRIGGER]}\\[2pt]
\texttt{Write a narrower USE WHEN condition that still covers the correct examples and excludes the over-triggered ones. Express it as a minimal keyword or structural pattern (e.g., "contains X OR contains Y").}\\
\texttt{Narrowed USE WHEN:}
\end{quote}

Each candidate narrowing is evaluated on the locked task set. The narrowing that maximizes net gain while maintaining zero observed performance degradations is selected.

\paragraph{Scoring guard and retry logic.}
Each model call goes through a scoring guard that checks whether the returned answer can be parsed to a valid option or numeric value. Failed parses trigger up to two retries with a slightly relaxed extraction prompt before the item is marked as a parse failure. Items with parse failures are excluded from paired comparisons but logged. The low-concurrency call schedule (sequential rather than batch-parallel) was adopted to reduce inter-call variance across conditions; all conditions for a given item are scored in the same session where possible.

\paragraph{Paired bootstrap.}
The 95\% CI and performance degradation counts are computed over 10,000 paired bootstrap resamples of the item-level effect vector $e_i = \mathbf{1}[\text{\method{} correct}] - \mathbf{1}[\text{\csicl{} correct}]$. A performance degradation is defined as $e_i = -1$; a fix as $e_i = +1$; a no-op as $e_i = 0$. The reported CI is the percentile interval of the bootstrap distribution of $\sum_i e_i / n$.

\section{Baseline Method Execution Details}
\label{app:baselines}

\paragraph{Shared protocol.}
Both baselines follow the same seeded, leakage-controlled protocol. The seed prompt is the locked gpt-4.1-mini \csicl{} cheat sheet for each task, i.e., the same anchor \method{} refines. Optimization touches only the 150-item train split per task; every optimization-time prompt and response is logged and scanned by an automated audit for verbatim test-item text, with zero matches on all four tasks for both methods. The single best prompt per task is frozen (SHA-256 recorded) before any test call, then evaluated once on the task's 100 test items through the locked evaluator path (temperature 0.0, seed 42, \texttt{max\_tokens} 64, gpt-4.1-mini, unchanged prompt wrapper and answer parser). Per-item fixes and regressions are computed against the same frozen \csicl{} anchor cache as the main result, and pooled deltas use a task-stratified paired bootstrap with 10,000 resamples. The verdict parse rate is 100\% (400/400) for both methods.

\paragraph{ProTeGi.}
We run the LMOps \texttt{prompt\_optimization} implementation of ProTeGi (textual-gradient beam search) with its documented default configuration: 4 optimization rounds, beam size 4, minibatch size 32, 4 gradients per round with 4 errors per gradient, up to 8-fold prompt expansion, and the UCB bandit evaluator (8 evaluation rounds, 8 prompts per round, 32 samples per evaluation, $c{=}1.0$). \texttt{gpt-4.1-mini} serves as both the task model and the gradient model; optimization decoding uses temperature 0.0 with a fixed run seed. The candidate that ProTeGi mutates is the cheat-sheet text itself, and every optimization-time scoring call wraps the candidate in the locked evaluator template, so optimization observes the exact deployment presentation. Because the stock implementation selects its final prompt by test-set score, we modified final selection to use full-train accuracy over the final beam; no test items are read anywhere in the optimization loop. Optimization used 28{,}816 model calls across the four tasks. On disambiguation QA the beam search retained the seed cheat sheet unchanged (hash-identical) after all rounds.

\paragraph{GEPA.}
We run the standalone \texttt{gepa} library (v0.1.1) through its public \texttt{optimize()} API with its default adapter, supplying the seed cheat sheet as the system-prompt candidate. The 150 train items are shuffled once with a fixed seed and split 75/75 into the library's internal train and validation sets; the metric is exact-match correctness through our own task parsers. Because the library requires an explicit budget, we set \texttt{max\_metric\_calls} to 680 per task, the value produced by the DSPy GEPA ``light'' preset budget formula; all other settings are library defaults, including default sampling on optimization-time calls. \texttt{gpt-4.1-mini} serves as both the task model and the reflection model. GEPA returned the seed prompt unchanged on three of four tasks (every reflective mutation scored at or below the seed on its own validation split), so its test-time deltas on those tasks measure fresh-call variance against the cached anchor rather than a prompt change; its one adopted rewrite (object counting) yielded 11 fixes and 12 regressions.

\section{\diag{} Diagnostic Pipeline}
\label{app:rfr}

The \diag{} pipeline is the source-only refinement baseline we used to produce the diagnostic results in Table~\ref{tab:diag}. It runs in three phases on a source model over a training split, with each phase accepting patches only when they pass source-model quality gates with no multi-model evaluation at any stage. This is the fundamental reason its output becomes model-specific. Below we describe each phase in detail.

\paragraph{Phase 0: Bootstrap cheat sheet.}
The pipeline samples the first 75 training examples (configurable). For examples lacking an explicit chain-of-thought, it first generates zero-shot reasoning traces via the source model. The examples are then formatted as Question / Reasoning / Answer triplets and passed to the following prompt:

\begin{quote}\small
\texttt{Create a cheat sheet based on the examples below. You will be asked to answer questions similar to these examples during the test, without being allowed to refer to the examples at that time. Your task here is to make a cheat sheet that will help you answer such problems correctly. First, carefully read the examples below and identify which ones you find most difficult to answer.}\\[2pt]
\texttt{[EXAMPLES]}\\[2pt]
\texttt{Now, create a cheat sheet to help you solve the difficult examples. Exclude any content that is easy for you, and only include specific, detailed points to address the challenging ones.}
\end{quote}

The output is a \texttt{prior\_knowledge} text block that forms the starting anchor for Phase 1.

\paragraph{Phase 1: Iterative prior-knowledge patching.}
Phase 1 runs an iterative repair loop over source-model failures under the current \texttt{prior\_knowledge}. Each iteration (up to a configurable maximum, default 4) samples up to 15 failures and constructs a \emph{failure block} for each: the input, the expected and predicted answers, the source model's wrong reasoning trace, and — when oracle traces are available — the correct reasoning from a stronger reference model. The patch prompt is:

\begin{quote}\small
\texttt{You are refining a knowledge guide that helps a model answer questions correctly.}\\[2pt]
\texttt{=== CURRENT KNOWLEDGE GUIDE ===}\\
\texttt{[PRIOR\_KNOWLEDGE]}\\
\texttt{=== END KNOWLEDGE GUIDE ===}\\[2pt]
\texttt{The model is making the following errors:}\\
\texttt{[FAILURE\_BLOCKS]}\\[2pt]
\texttt{Produce an IMPROVED version of the knowledge guide that helps avoid these mistakes. You may:}\\
\texttt{- ADD a new rule, clarification, or concrete example}\\
\texttt{- MODIFY an existing rule to be more precise or correct}\\
\texttt{- REMOVE a rule that is actively causing errors}\\[2pt]
\texttt{Requirements: preserve rules that are working correctly; focus on ABSTRACT REASONING PRINCIPLES, not memorization of specific examples; output ONLY the improved knowledge guide text.}
\end{quote}

A candidate patch is committed only when it passes two source-model gates: (a) it fixes at least 20\% of the sampled failures (\emph{fix-rate gate}), and (b) it does not regress more than 20\% of a sample of currently-correct items (\emph{performance degradation gate}). The loop exits when source accuracy reaches 85\% or when two consecutive iterations produce no accepted patch. The oracle contrast — showing the correct reasoning from a stronger model alongside the wrong reasoning — provides a concrete target for the patcher, but is itself derived from source-model failures and carries no cross-model validity signal.

\paragraph{Phase 2: Case study generation.}
Phase 2 operates on residual failures after Phase 1. It partitions failures by structural feature signature (e.g., argument form, entity type, geometric relation) and, for each partition with at least 3 failures, generates up to 3 candidate case studies using a task-specific structured prompt. Each generated case study has the following fields:

\begin{quote}\small
\texttt{=== CASE STUDY: [title] ===}\\
\texttt{FAILURE\_TYPE: A (missing knowledge) | B (wrong pattern)}\\
\texttt{ACTIVATE IF:}\\
\texttt{  - [structural condition 1]}\\
\texttt{  - [structural condition 2]}\\
\texttt{DO NOT ACTIVATE IF: [boundary condition]}\\
\texttt{COMMON WRONG MOVE: [what the model typically does wrong]}\\
\texttt{NEXT CHECK: [mechanical verification step]}\\
\texttt{WHY THIS WORKS: [1--2 sentences explaining the fix]}\\
\texttt{SUPPORT:}\\
\texttt{  • E1 = ... | E2 = ... | Answer: ... — [note]}
\end{quote}

The ACTIVATE-IF conditions become the activation predicate; DO NOT ACTIVATE IF becomes the boundary predicate. A candidate case study is accepted when it (a) fixes at least 30\% of the partition failures, (b) does not regress more than 15\% of correct items, and (c) is not a duplicate or near-duplicate of an existing case study through an LLM judge. Accepted case studies are appended to the cheat sheet and re-scored; partitions are refreshed after each flush, and the loop runs for up to 5 iterations.

\paragraph{Why model-specificity accumulates.}
All three gates in Phases 1 and 2 are evaluated against the \emph{source model only}. A Phase 1 patch is accepted if it fixes source failures and does not break source-correct items. A Phase 2 case study is accepted if its partition fix rate and performance degradation rate are satisfactory on the source. Neither the fix-rate threshold nor the degradation threshold is measured against any other model other than the source model at any point during development. As a result, the accepted patches and case studies are precisely those that explain the source model's residual error profile, which may not align with the shared error profile across models. Through this mechanism, the \diag{} artifact accumulates model-specific content as it progresses through Phase 1 and Phase 2.

\section{Supplementary Audits}
\label{app:audits}
\paragraph{Protocol.}
The repaired rerun used the same item universe, \csicl{} cache, parser, prompt builder, retry path, scoring guard, and inactive passthrough rule as the locked protocol. It activated 39 \method{}+\csicl{} item-rule applications and had zero missing effect rows, baseline mismatches, parse errors, or failed calls (Table~\ref{tab:audit}).

\begin{table}[h]
\centering
\small
\begin{tabular}{lr}
\toprule
Audit field & Value \\
\midrule
Main eligible items & 400 \\
Full18 items & 1716 \\
Activated applications & 39 \\
Missing effect rows & 0 \\
Baseline mismatches & 0 \\
Parse errors & 0 \\
Failed calls & 0 \\
Bootstrap samples & 10,000 \\
\bottomrule
\end{tabular}
\caption{Protocol audit for the repaired rerun.}
\label{tab:audit}
\end{table}

\begin{table*}[t]
\centering
\footnotesize
\begin{tabular}{llrrr@{\qquad}llrrr}
\toprule
Task & Status & $n$ & \csicl{} & \method{} & Task & Status & $n$ & \csicl{} & \method{} \\
\midrule
boolean & no \atom{} & 100 & 89.00 & 89.00 & causal judgement & no \atom{} & 87 & 66.67 & 66.67 \\
date understanding & no \atom{} & 100 & 72.00 & 72.00 & disambig. QA & main & 100 & 83.00 & 87.00 \\
formal fallacies & main & 100 & 67.00 & 67.00 & geometric shapes & main & 100 & 56.00 & 63.00 \\
LD five & no \atom{} & 100 & 73.00 & 73.00 & LD seven & no \atom{} & 100 & 73.00 & 73.00 \\
LD three & no \atom{} & 100 & 96.00 & 96.00 & navigate & no \atom{} & 100 & 78.00 & 78.00 \\
object counting & main & 100 & 68.00 & 68.00 & penguins table & no \atom{} & 58 & 74.14 & 74.14 \\
colored objects & no \atom{} & 100 & 81.00 & 81.00 & snarks & no \atom{} & 71 & 85.92 & 85.92 \\
sports & no \atom{} & 100 & 96.00 & 96.00 & temporal seq. & no \atom{} & 100 & 100.00 & 100.00 \\
tracking shuffled & no \atom{} & 100 & 38.00 & 38.00 & web of lies & excluded & 100 & 65.00 & 65.00 \\
\midrule
\multicolumn{5}{l}{Passthrough subtotal (14 tasks): $n{=}1316$, 77.74 / 77.74} & \multicolumn{5}{l}{Full 18 total: $n{=}1716$, 75.58 / 76.22} \\
\bottomrule
\end{tabular}
\caption{Full 18 task set: eligibility status and per-task accuracy (\csicl{} / \method{}). Tasks without an eligible \atom{} are exact passthrough, so the two accuracies are identical by construction. Web of lies is excluded because its candidate \atoms{} used activation boundary derived from test items, and no \atom{} is deployed for it.}
\label{tab:eligibility}
\end{table*}

\paragraph{AGIEval anchor diagnostic and rebuild check.}
AGIEval is excluded from the main claim because the available \csicl{} anchor underperforms raw prompting on all three tasks (Table~\ref{tab:agieval}, top). A \csicl{} anchor that is worse than raw prompting is not a valid starting point for a distillation improvement claim: any \method{} gain measured against it is at high risk of being confounded by the anchor deficit rather than attributable to the \atoms{}. We therefore ran a development-only anchor rebuild check (E8) to determine whether a nonnegative anchor could be recovered for any task (Table~\ref{tab:agieval}, bottom). Two tasks (LogiQA and LSAT-LR) produced nonnegative development anchors, but LSAT-AR found no variant that satisfied the nonnegative-anchor requirement. As a result, AGIEval remains diagnostic and is not promoted to the main benchmark task set.

\begin{table}[h]
\centering
\small
\begin{tabular}{lrrrr}
\toprule
Task & $n$ & Raw & \csicl{} & Net \\
\midrule
LogiQA & 300 & 57.67 & 57.36 & $-$1 \\
LSAT-AR & 23 & 40.00 & 30.43 & $-$11 \\
LSAT-LR & 373 & 81.57 & 79.22 & $-$6 \\
Aggregate & 696 & 63.51 & 60.92 & $-$18 \\
\bottomrule
\end{tabular}
\vspace{0.5em}

\begin{tabular}{lrrrrl}
\toprule
Task & Raw dev & Anchor dev & Fix & Reg. \\
\midrule
LogiQA & 50.00 & 70.00 & 4 & 0 \\
LSAT-AR & -- & -- & -- & -- \\
LSAT-LR & 85.00 & 90.00 & 1 & 0 \\
\bottomrule
\end{tabular}
\caption{Top: AGIEval rejected-anchor diagnostic. Bottom: development-only nonnegative-anchor rebuild check (E8).}
\label{tab:agieval}
\end{table}

\paragraph{Leave-one-family-out proxy diagnostic.}
Table~\ref{tab:loo} evaluates whether the repaired \atom{} set remains useful when a target model family is treated as held out during proxy gating. This is a retrospective diagnostic: the \atom{} set is fixed from the source protocol and is not regenerated with each family excluded. It should therefore not be read as a full end-to-end leave-one-family-out experiment, but rather as a check on whether the accepted \atoms{} are source-specific to the extent of actively hurting other families even under a favorable gate. The results show positive net transfer for Anthropic, Mistral, and Qwen, while Gemini incurs a net performance degradation. The performance degradation for Gemini is consistent with the cross-model result in Table~\ref{tab:crossmodel} and reflects that the \atom{} activation boundaries were derived from source-model failures without Gemini-specific validation.

\begin{table}[h]
\centering
\small
\resizebox{\columnwidth}{!}{%
\begin{tabular}{lrrrrr}
\toprule
Held-out family & \csicl{} & \method{}+CS & Fix & Reg. & Net \\
\midrule
Anthropic & 75.75 & 76.25 & 3 & 1 & $+$2 \\
Google/Gemini & 64.25 & 63.75 & 3 & 5 & $-$2 \\
Mistral & 37.25 & 40.75 & 14 & 0 & $+$14 \\
Qwen & 51.00 & 53.75 & 13 & 2 & $+$11 \\
\bottomrule
\end{tabular}}
\caption{Retrospective leave-one-family-out proxy-gating diagnostic. The \atom{} set is frozen from the source protocol; this is not a full end-to-end LOO experiment. Positive net transfer holds for three of four families, with Gemini as the exception.}
\label{tab:loo}
\end{table}

\paragraph{Per-item mechanism examples.}
Table~\ref{tab:item_examples} gives examples of fixes, no-op passthrough, and boundary blocking. These examples are illustrative and are not used as additional evaluation evidence.
\begin{table}[h]
\centering
\footnotesize
\resizebox{\columnwidth}{!}{%
\begin{tabular}{llllllll}
\toprule
Task & ID & CS & \method{} & Gold & Atom & Blocked & Outcome \\
\midrule
geo & 0155 & B & G & G & geo-6 & no & fix \\
geo & 0203 & A & K & K & geo-1 & no & fix \\
disambig & 0171 & C & B & B & disambig & no & fix \\
disambig & 0242 & C & A & A & disambig & no & fix \\
obj. count & 0150 & 9 & 9 & 9 & none & n/a & no-op \\
disambig & 0154 & A & A & A & disambig & yes & boundary no-op \\
\bottomrule
\end{tabular}}
\caption{Qualitative per-item mechanism examples. The no-op rows illustrate inactive passthrough and boundary blocking rather than global prompt expansion.}
\label{tab:item_examples}
\end{table}

\begin{table*}[ht]
\centering
\small
\begin{tabular}{lrrr|rrr}
\toprule
 & \multicolumn{3}{c|}{Causal judgment} & \multicolumn{3}{c}{Object counting} \\
\cmidrule(lr){2-4}\cmidrule(lr){5-7}
Target model & App. & Hol. & $\Delta$ & App. & Hol. & $\Delta$ \\
\midrule
\texttt{gpt-4.1}        & \textbf{62.8} & 61.3          & $-$1.5           & \textbf{83.7} & 68.0          & $-$15.7 \\
\texttt{llama-3.3-70b}  & \textbf{66.3} & 64.4          & $-$1.9           & 60.3          & \textbf{62.0} & $+$1.7  \\
\texttt{gemini-2.5-fl}  & 61.3          & \textbf{65.9} & $+$4.6           & \textbf{99.0} & 97.0          & $-$2.0  \\
\texttt{claude-3.5-hk}  & 64.0          & \textbf{65.1} & $+$1.1           & 89.7          & \textbf{92.0} & $+$2.3  \\
\texttt{mistral-7b}     & \textbf{61.3} & 54.8          & $-$6.5           & 35.3          & \textbf{47.0} & $+$11.7 \\
\texttt{qwen2.5-7b}     & 57.9          & 57.9          & $\phantom{+}$0.0 & 36.3          & \textbf{60.3} & $+$24.0 \\
\bottomrule
\end{tabular}
\caption{Exploratory cross-model transfer accuracy (\%, mean over seeds 1000/2000/3000) for append vs.\ holistic rewrite on causal judgment and object counting. Bold marks the better mode per cell; $\Delta$ = holistic $-$ append (pp).}
\label{tab:rewrite}
\end{table*}

\section{Exploratory Holistic Rewrite Results}
\label{app:holistic}
Iterative holistic rewrite is an exploratory alternative to \atom{} appending. Rather than appending a routed \atom{} alongside the \csicl{} anchor, holistic rewrite integrates the repair directly into the cheat sheet text, restructuring or merging rules where needed. It may help when a repair requires structural integration rather than a localized addition, but it is not yet robust enough to be a main contribution. Table~\ref{tab:rewrite} reports held-out accuracy for append versus holistic rewrite on causal judgment (CJ) and object counting (OC) across six target models (mean over seeds 1000/2000/3000).

On OC, holistic rewrite yields large gains for the two small open-weight models (\texttt{mistral-7b}: $+11.7$ pp, \texttt{qwen2.5-7b}: $+24.0$ pp) and modest gains for \texttt{claude-3.5-haiku} ($+2.3$ pp) and \texttt{llama-3.3-70b} ($+1.7$ pp), but substantially hurts \texttt{gpt-4.1} ($-15.7$ pp). On CJ, the picture is more mixed: holistic rewrite helps Gemini and Claude but hurts Mistral and is neutral on Qwen. The pattern suggests that holistic rewrite benefits models whose reasoning is less sensitive to prompt structure, particularly smaller open-weight models that may rely more on the integrated rule text. On the other hand, stronger proprietary models are harmed when the cheat sheet structure is disrupted. These results are promising for structured-integration use cases but should be treated as exploratory until replicated with tighter controls.

\end{document}